\documentclass{article}
\usepackage[utf8]{inputenc}
\usepackage{amsmath,amsfonts,bm}

\def\eqref#1{equation~\ref{#1}}
\def\1{\bm{1}}

\def\mA{{\bm{A}}}

\def\mW{{\bm{W}}}
\def\mX{{\bm{X}}}

\DeclareMathAlphabet{\mathsfit}{\encodingdefault}{\sfdefault}{m}{sl}
\SetMathAlphabet{\mathsfit}{bold}{\encodingdefault}{\sfdefault}{bx}{n}

\newcommand{\sigmoid}{\sigma}

\newcommand{\newmodel}{\text{SI-CIFG}}
\newcommand{\newsigmoid}{\text{SI-}\sigma}
\newcommand{\newtanh}{\text{SI-}\tanh}
\newcommand{\newattn}{\text{SI-Attn}}

\title{Efficient Language Model Architectures for Differentially Private Federated Learning}

\author{
  Jae Hun Ro\\
  \texttt{jaero@google.com}\\
  Google, Inc.
  \and
  Srinadh Bhojanapalli\\
  \texttt{bsrinadh@google.com}\\
  Google, Inc.
    \and
  Zheng Xu\\
  \texttt{xuzheng@google.com}\\
  Google, Inc.
    \and
  Yanxiang Zhang\\
  \texttt{zhangyx@google.com}\\
  Google, Inc.
    \and
  Ananda Theertha Suresh\\
  \texttt{theertha@google.com}\\
  Google, Inc.
}

\usepackage{subfigure}

\usepackage{natbib}
\usepackage{graphicx}
\usepackage{fullpage}
\usepackage{hyperref}
\hypersetup{colorlinks,linkcolor={red},citecolor={blue},urlcolor={red}}
\usepackage{amsmath, amssymb, amsthm, algorithm,pseudo}
\usepackage{xcolor}
\newcommand{\ignore}[1]{{}}

\newtheorem{theorem}{Theorem}

\newcommand{\conf}[1]{}

\theoremstyle{plain}
\newtheorem{proposition}[theorem]{Proposition}

\begin{document}

\maketitle

\begin{abstract}
Cross-device federated learning (FL) is a technique that trains a model on data distributed across typically millions of edge devices without data leaving the devices.
SGD is the standard client optimizer for on device training in cross-device FL, favored for its memory and computational efficiency.
However, in centralized training of neural language models, adaptive optimizers are preferred as they offer improved stability and performance.
In light of this, we ask if language models can be modified such that they can be efficiently trained with SGD client optimizers and answer this affirmatively.

We propose a scale-invariant \emph{Coupled Input Forget Gate} (SI CIFG) recurrent network by modifying the sigmoid and tanh activations in the recurrent cell 
and show that this new model converges faster and achieves better utility than the standard CIFG recurrent model in cross-device FL in large scale experiments.
We further show that the proposed scale invariant modification also helps in federated learning of larger transformer models.
Finally, we demonstrate the scale invariant modification is also compatible with other non-adaptive algorithms.
Particularly, our results suggest an improved privacy utility trade-off in federated learning with differential privacy.
\end{abstract}

\section{Introduction}
\label{introduction}

Federated learning (FL) is a technique that trains a model on data distributed across devices without data leaving the device \citep{Konecn2016FederatedLS,pmlr-v54-mcmahan17a}.
FL has been applied in a variety of diverse settings, including language-based applications \citep{Hard2018FederatedLF,chen-etal-2019-federated,49232,9084352,49696}.
Specifically, we examine cross-device FL~\citep{kairouz2021advances}, where local clients are edge devices with limited resources and computing power, which can number in the millions.
Previous works on language modeling in cross-device FL often use small recurrent-based models of less than $10$M parameters \citep{Hard2018FederatedLF,reddi2020adaptive,xu-etal-2023-federated}, while more recent works leverage a variety of efficient techniques for training larger Transformer-based models \citep{Hilmkil2021ScalingFL,ro-etal-2022-scaling}.
In this work, we investigate modular strategies applicable to various model architectures for improving training of both small and large models in cross-device FL.

Existing works on improving FL usually focus on developing better optimizers FedAvg, FedProx, Mime, FedDyn etc \citep{li2020federated,50448,NEURIPS2021_f0e6be4c,acar2021federated}.
While advanced optimizers are typically used in the server, (e.g., in the optimizer FedAdam, Adam optimizer is used in the server), in practice, the preferred client optimizer is often SGD for its memory efficiency. Note that using an adaptive optimizer like Adam in clients requires storing first and second moments of gradients, which improves the memory requirement considerably.
However, neural language models such as recurrent LSTMs \citep{yu2019review} or Transformers \citep{vaswani2017}, 
typically require more memory intensive adaptive optimizers, such as Adagrad or Adam that store both the first and second moment of gradients, and suffer in performance when trained with SGD~\citep{zhang2020adaptive}. 
Hence we ask the question: Can we achieve the best of both worlds and effectively train expressive architectures with memory efficient optimizers for language modeling in FL?

\citet{pmlr-v162-li22b} studied this question in the context of training centralized Transformer encoder models, such as BERT, and proposed using Scale Invariant Transformers for improved optimization of Transformers using SGD.
Using Scale Invariant Transformers, they were able to use SGD to obtain a similar performance to that of standard Transformers using the Adam optimizer.
Naturally, this raises the question if there exists scale invariant version of other neural architectures, e.g. LSTMs that can be optimized well with simple SGD.

Federated learning can also be combined with other privacy techniques to provide strong privacy protection to various threat models~\citep{zhang2023private,bonawitz2022federated}.
Differential privacy (DP)~\citep{dwork2006calibrating} is a statistical framework that provides rigorous guarantees for privacy protection and is adopted in federated learning to prevent models from memorizing individual information~\citep{mcmahan2017learning,ramaswamy2020training, el2022differential, wei2020federated, girgis2021shuffled}.
More recently, by applying the family of DP-Follow The Regularized Leader (DP-FTRL) algorithms~\citep{pmlr-v139-kairouz21b,choquette2023amplified} that have strong privacy-utility trade-offs without relying on sampling assumptions, meaningful formal differential privacy guarantees have been achieved for production language models in practical cross-device systems~\citep{xu-etal-2023-federated}.

\section{Our contributions}

\textbf{Improving LSTM architectures for FL.}
Long Short-Term Memory (LSTM)~\citep{hochreiter1997long} language models are often used in large scale FL studies due to their small size  \citep{pmlr-v54-mcmahan17a,47586,Hard2018FederatedLF,pmlr-v139-kairouz21b,xu-etal-2023-federated}. In particular, \cite{Hard2018FederatedLF} proposed to use Coupled Input Forget Gate (CIFG) LSTMs for federated learning for mobile keyboard predictions for its improved parameter and computational efficiency over the vanilla LSTM.
Motivated by this, we develop a novel scale-invariant CIFG model (\newmodel) with modified activation functions for FL.

\textbf{Application to FL.}
In cross-device FL, each client typically runs multiple steps of local SGD on their local data to produce model parameter updates.
These updates are then typically combined at the server with a federated optimizer such as FedAdam \citep{50448}.
This raises an important question: does our \newmodel~offer any advantages in this setting where one of the optimizers is SGD and the other is an adaptive optimizer like Adam?
We show that this is indeed the case and that both our proposed \newmodel~as well as the already existing scale-invariant Transformer \citep{pmlr-v162-li22b} (SI Transformer), using scale-invariant attentions, perform significantly better than their standard counterparts on a variety of experiments by improving convergence speeds in large scale FL experiments, while remaining robust to higher learning rates and heterogeneous networks.

\textbf{Training with differential privacy.}
FL models are trained with differential privacy using the DP-FTRL algorithm \citep{pmlr-v139-kairouz21b}.
In this scenario, while the local steps are still carried out via SGD, the model updates from clients are additionally clipped and aggregated with noise
at the server.
We show that scale invariant models also outperform their standard counterparts on experiments in a large-scale FL system with differential privacy.

\section{Scale Invariant Architectures}
\label{architectures}

\subsection{Previous scale invariant architectures}

In this section, we briefly review Scale Invariant Transformers~\citep{pmlr-v162-li22b}.
Recall that a function $f$ is scale invariant if $f(ax) = f(x)$ for any scalar $a > 0$. 
Let $n$ be the input sequence length and $d$ be the hidden dimension of the Transformer model.
Recall that for a given input $\mX \in \mathbb{R}^{d \times n}$, a Transformer computes self attention as follows:

\begin{equation}\label{eq:attn}
    \text{Attn}(\mX) =  \text{SoftMax} \left( (\mW_Q \mX)^\top . \mW_K \mX\right).
\end{equation}

Here $\mW_Q$ and $\mW_K$ are the Query and Key projections, respectively.
This operation is not scale invariant, as scaling the weights ($\mW_Q , \mW_K$) changes the output attention probabilities. \citet{pmlr-v162-li22b} proposed the following alternative attention computation:

\begin{equation}\label{eq:st_attn}
    \newattn(\mX) =  \text{N}\left(\text{ReLU} \left( (\mW_Q \mX)^\top . \mW_K \mX\right)\right).
\end{equation}

Here, N is the row-wise normalization operator - $\text{N}(\mA)_{ij} = \frac{\mA_{ij}}{\sum_j \mA_{ij}}$.
In particular, \citet{pmlr-v162-li22b} replaced the softmax in attention computation, with the ReLU activation followed by row-wise normalization.
This modifies the attention computation to be scale invariant.
They further modify the Transformer to be a Pre-LN activation model and use ReLU activation instead of GeLU in the feedforward layers.
We use the same architecture in our experiments.

However, \citet{pmlr-v162-li22b} tested their method only on centralized encoder models (BERT).
In this paper, we will extend the results to decoder-only Transformers trained using a language modeling objective in cross-device FL.

\subsection{New scale invariant architectures}
Inspired by the Scale Invariant Transformer, we now design a novel Scale Invariant version of the CIFG architecture we call \newmodel.
We note that the same changes from the Scale Invariant Transformer do not apply to the CIFG as due to architecture differences, scale sensitivity arises from different functions for CIFG models.

We focus on CIFG networks for their improved parameter and computational efficiency over the vanilla LSTMs.
The CIFG network uses a single gate to control self-connections in both input and recurrent cells, which reduces the number of parameters per cell by 25\% \citep{hochreiter1997,cho-etal-2014-learning,7508408}.
The shared gates increase efficiency, with little to no impact on quality, which is critical in the typically resource constrained edge device environment of cross-device FL.
Moreover, we expect that our proposed changes can also be directly applied to the LSTM model.

First, we review the basic CIFG before our proposed architecture changes.
Recall that for a given time step $t$ and input $x_t \in \mathbb{R}^d$, the CIFG forward pass can be written as follows:

\begin{align*}
& f_t = \sigmoid(W_f x_t + U_f h_{t-1} + b_f) & \textit{forget gate} \\
& i_t = 1 - f_t & \textit{coupled input forget gate} \\
& o_t = \sigmoid(W_o x_t + U_o h_{t-1} + b_o) & \textit{output gate} \\
& c_t = f_t \odot c_{t-1} + i_t \odot \tanh(W_c x_t + U_c h_{t-1} + b_c) & \textit{cell state} \\
& h_t = o_t \odot \tanh(c_t)
\end{align*}

where $d$ and $h$ are the input and hidden dimensions, respectively, and $W \in \mathbb{R}^{h \times d}$, $U \in \mathbb{R}^{h \times h}$, and $b \in \mathbb{R}^h$ are the cell's trainable weight and bias parameters. Here $\sigmoid$ and $\tanh$ are Sigmoid and Tanh activation functions, respectively.
This architecture is sensitive to input scale, mainly because of the non-linearities in the $\sigmoid$ and 
$\tanh$ activations.
We first propose modifying the activation functions to be scale invariant by replacing $\sigmoid$ with Relu and $\tanh$ with linear activation.
However, this no longer guarantees that intermediate outputs of different gates are normalized.
To further ensure that the intermediate features are normalized we propose using a Max-Normalization - \textsc{MaxN}, which normalizes each entry of the feature vector using its max absolute value along the hidden dimension.
Formally,
\begin{equation}
    \textsc{MaxN}(x)_i = \frac{x_i}{\max_{j \in [d]} |x_j|}.
\end{equation}

Based on this, we propose the following scale invariant replacement for $\sigmoid$ activation.
\begin{equation}\label{eq:si_sigmoid}
    \newsigmoid(x)_i = \textsc{MaxN}(\text{Relu}(x))_i = \frac{\text{Relu}(x)_i}{\max_{j \in [d]}(\text{Relu}(x)_j)}.
\end{equation}

Similarly, we also propose a scale invariant version of $\tanh$.

\begin{equation}\label{eq:si_tanh}
    \newtanh(x)_i = \textsc{MaxN}(x)_i = \frac{x_i}{\max_{j \in [d]}(|x_{j}|)}.
\end{equation}

It is straightforward to see that both $\newsigmoid$ and $\newtanh$ are scale invariant functions and we provide a short proof for completeness.

\begin{proposition}
Both $\newsigmoid$ and $\newtanh$ are scale invariant functions.
\end{proposition}
\begin{proof}
Let $a > 0$. Then 
for any $i \in d$, $\text{Relu}(a x)_i = a \text{Relu}(x)_i $ and hence,
\begin{align*}
\newsigmoid(a x)_i & = \frac{\text{Relu}(a x)_i}{\max_{j \in [d]}(\text{Relu}(a x)_j)} 
 = \frac{a \text{Relu}(x)_i}{a \max_{j \in [d]}(\text{Relu}(x)_j)} 
 = \frac{\text{Relu}(x)_i}{\max_{j \in [d]}(\text{Relu}(x)_j)} 
 = \newsigmoid(x)_i. 
\end{align*}
The calculations for $\newtanh$ are are similar and omitted.
\end{proof}

\section{Experiments with Federated learning}
\label{experiments}

We report results for experiments using scale invariant architectures in large scale FL in both simulation and live production experiments.
For simulations, we train a language model on the English Stack Overflow federated dataset, containing questions and answers from the forum grouped by username, provided from TensorFlow Federated (TFF) \citep{tff}.
For live production experiments, we train an English language model on millions of virtual keyboard user devices and follow the same settings and FL requirements for client participation as \citet{Hard2018FederatedLF}.
All experiments were implemented using the open-source FedJAX \citep{fedjax2021} and TFF libraries.

\subsection{Federated experiments on public datasets}

For experiments on the Stack Overflow federated dataset, we compare the following models: 

\begin{itemize}
  \item CIFG $19$M: Coupled Input Forget Gate variant of LSTM with $19$M trainable parameters with $1$ layer of size $2048$, embedding size $1024$, and tied input and output embeddings \citep{press-wolf-2017-using}.
  \item \newmodel~$19$M: Modified CIFG $19$M using $\newsigmoid$ and $\newtanh$ activations.
  \item Transformer $21$M: Transformer with $21$M trainable parameters with $6$ layers, $8$ attention heads, MLP size $2048$, embedding size $512$, and tied input and output embeddings.
  \item SI Transformer $21$M: Modified Transformer $21$M using \newattn.
\end{itemize}

We use WordPiece \citep{Wu2016GooglesNM} for subword tokenization with a vocabulary size of $4K$ to avoid potential bottlenecks in embeddings for larger vocabularies.
For FL training, we use FedAdam \citep{50448} which uses Adam \citep{Kingma2014AdamAM} for the server optimizer and SGD for the client optimizer with the same settings used by \citet{50448}, with the exception of learning rates.
We then sweep over learning rates for each model with $5$ different random seeds for client sampling with $500$ clients per round for $3K$ communication rounds and maximum sequence length of $20$.
Details on speicifc hyperparameter settings and sweeps can be seen in Appendix \ref{app:simulation}.

\begin{figure}[ht]
\centering
\subfigure{\includegraphics[width=0.45\columnwidth]{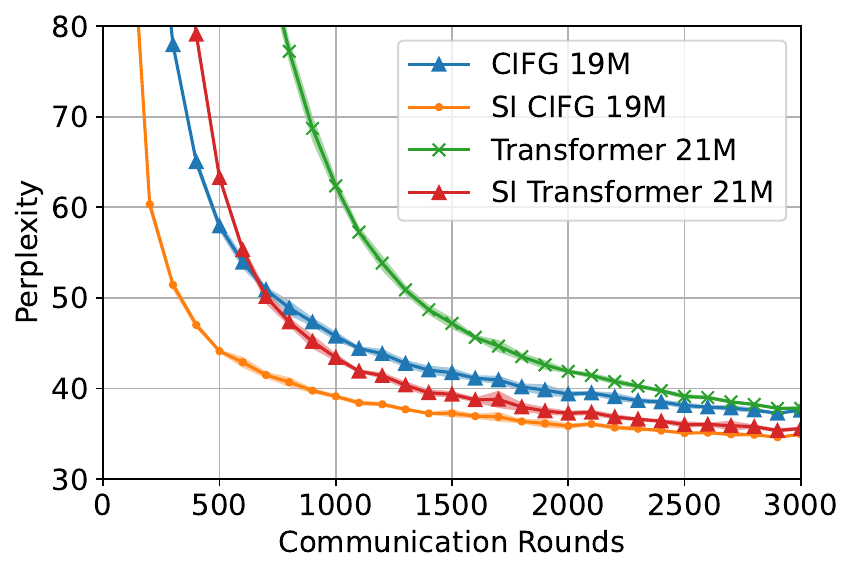}}
\subfigure{\includegraphics[width=0.45\columnwidth]{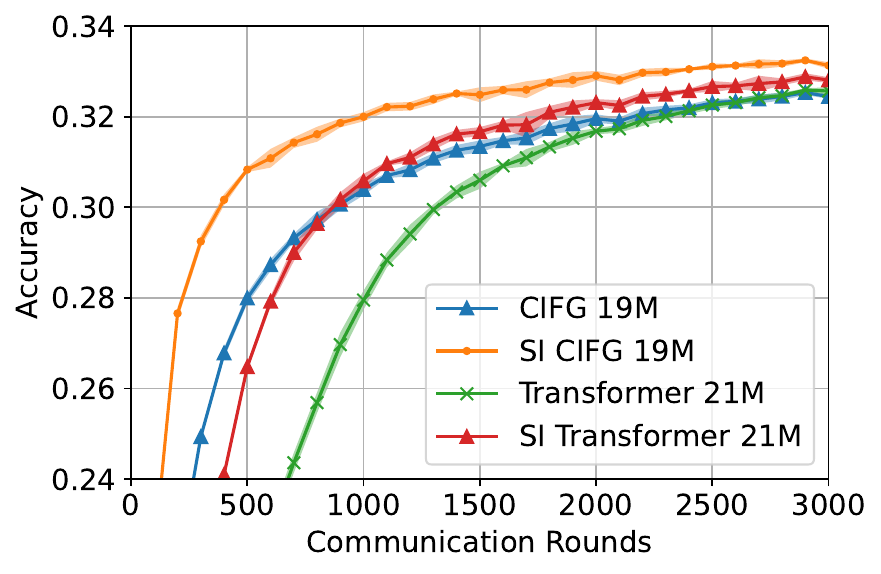}}
\caption{Perplexity and accuracy on the Stack Overflow test dataset with shading indicating standard deviation over $5$ random seeds.}
\label{fig:stackoverflow}
\end{figure}

We report perplexity and accuracy, discounting end-of-sequence
tokens, on the Stack Overflow test dataset over $3K$ communication rounds in Figure~\ref{fig:stackoverflow} with final values in Table~\ref{tab:stackoverflow} (Appendix~\ref{app:simulation}).
We observe that applying Scale Invariance significantly increases the rate of convergence for both the Transformer and CIFG, surpassing their respective base counterparts within $100$ communication rounds.
Our proposed \newmodel~yields the best final quality and has the fastest convergence speed by far.
We next continue to live production experiments, where the network of clients is much larger and more heterogeneous than simulation. 

\subsection{Live production experiments}

For live production experiments for cross-device FL on English virtual keyboard client devices, similar to \citet{Hard2018FederatedLF},
we compare the following models:

\begin{itemize}
  \item CIFG $9$M: CIFG with $9$M trainable parameters with $1$ layer of size $2048$, embedding size $512$, and tied input and output embeddings.
  \item \newmodel~$9$M: Modified CIFG $9$M using $\newsigmoid$ and $\newtanh$ activations.
  \item Transformer $11$M: Transformer with $11$M trainable parameters with $3$ layers, $8$ attention heads, MLP size $2048$, embedding size $512$, and tied input and output embeddings.
  \item SI Transformer $11$M: Modified Transformer $11$M using \newattn.
\end{itemize}

We use smaller sizes here compared to our previous simulation experiments due to stricter resource constraints on client devices \citep{Hard2018FederatedLF,Ro2021CommunicationEfficientAF}.
Additionally, we also apply stochastic $8$-bit uniform quantization \citep{NIPS2017_6c340f25,pmlr-v70-suresh17a} on the upload of model updates from client to server due to tighter communication bottlenecks on mobile devices.
We use Fast WordPiece \citep{song-etal-2021-fast} for subword tokenization with a vocabulary size of $4K$ as it has been shown to be faster than WordPiece, allowing for more steps of training within the maximum time limit allocated for client devices.
Again, we use the FedAdam algorithm with $500$ clients per round for $3K$ communication rounds with maximum sequence length of $20$.
For more details on hyperparameters, refer to Appendix \ref{app:live}.

\begin{figure}[ht]
\centering
\subfigure{\includegraphics[width=0.45\columnwidth]{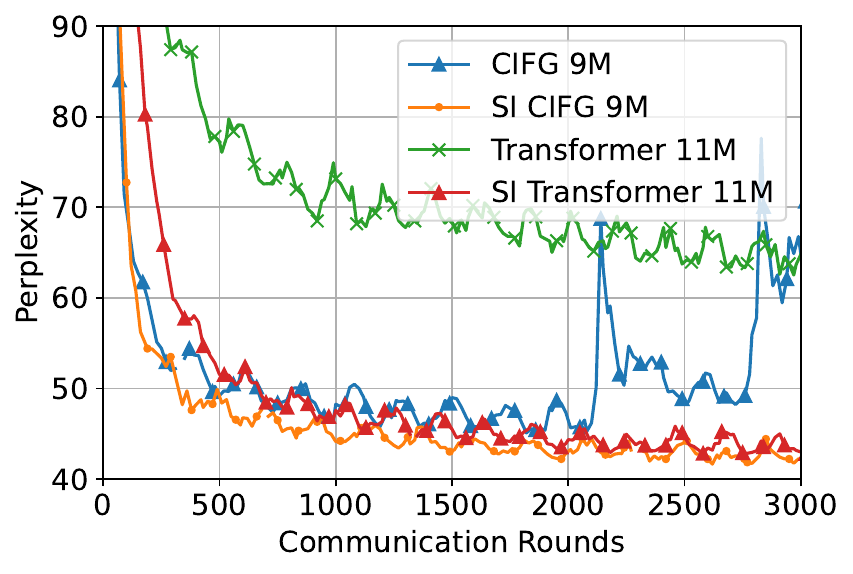}}
\subfigure{\includegraphics[width=0.45\columnwidth]{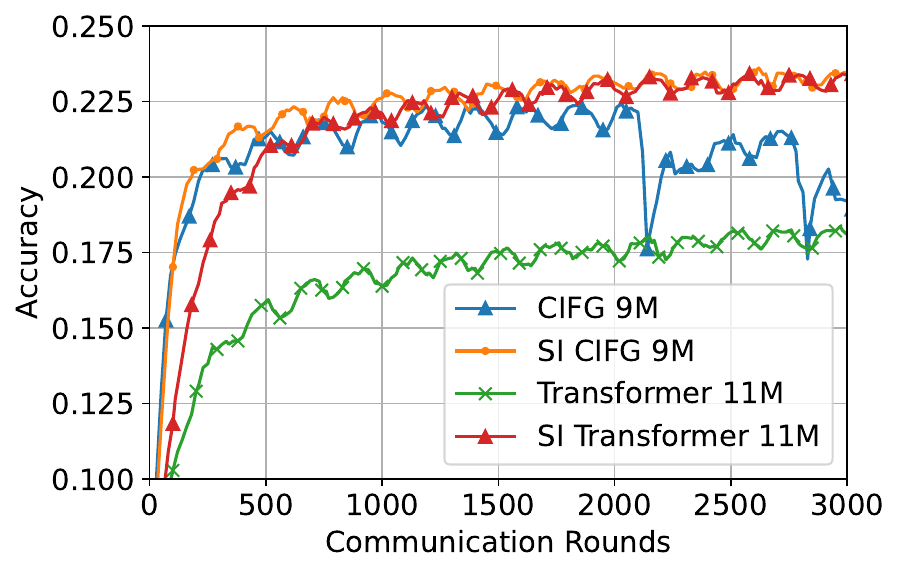}}
\caption{Perplexity and accuracy from live experiments on English virtual keyboard devices training from randomly initialized model weights.}
\label{fig:live-scratch}
\end{figure}

We report perplexity and accuracy for training the models from randomly initialized parameters on the population of English virtual keyboard client devices over $3K$ communication rounds in Figure~\ref{fig:live-scratch} with final values in Table~\ref{tab:live} (Appendix~\ref{app:live}).
The scale invariant architectures surpass their base counterparts within $100$ communication rounds and converge to significantly higher qualities.
While the base CIFG diverges in training at $2K$ rounds, which could be attributed to a number of potential issues \citep{pmlr-v28-pascanu13} when training recurrent models with SGD on client devices, our proposed \newmodel~trains smoothly, significantly outperforms the other models within $200$ rounds, and converges to the best final quality.
This improved training stability could be due to robustness to out-sized client updates in the $\newsigmoid$ and $\newtanh$ activations.

\section{Experiments with differentially private federated learning}

In this section, we apply our proposed scale invariant architectures to differentially private (DP) FL.
Specifically, we apply the DP variant of Follow-The-Regularized-Leader (DP-FTRL) \emph{Online TreeAgg} proposed by \citet{pmlr-v139-kairouz21b}.
For live production experiments, we train an English language model on millions of virtual keyboard user devices and mostly follow the same setup as \citet{xu-etal-2023-federated} for DP FL and compare the following models:

\begin{itemize}
  \item CIFG $6$M: CIFG with $6$M trainable parameters with $1$ layer of size $670$, embedding size $96$, and vocabulary size of ~$30K$.
  \item \newmodel~$6$M: Modified CIFG $6$M using $\newsigmoid$ and $\newtanh$ activations.
\end{itemize}

For training, we use $6500$ clients per round and the same noise multiplier of $7.0$ for $3K$ communication rounds with maximum sequence length of $10$ with word tokenization using a vocabulary size of ~$30K$.
The client optimizer is SGD with learning rate of $0.5$ and the server optimizer is SGD with momentum with learning rate $1.0$ and momentum $0.9$.
We set the noise multiplier in the DP-FTRL algorithm to obtain a z-CDP privacy of $1.05$.
We refer readers to \citet{bun2016concentrated} for the definition of z-CDP and \citet{pmlr-v139-kairouz21b} for the privacy guarantee calculations.
For more details and hyperparameter configurations, refer to Appendix \ref{app:live-dp}.
Before applying DP FL training, we first pre-train the models on the public English Colossal Clean Crawled Corpus (C4) \citep{2019t5} dataset for $370K$ steps and start DP FL training from the pre-trained checkpoint.
We report perplexity and in-vocab-accuracy, discounting out-of-vocabulary and end-of-sequence tokens, for DP FL training on the population of English virtual keyboard client devices over $3K$ communication rounds in Figure~\ref{fig:live-dp} with final values in Table~\ref{tab:live-dp} (Appendix~\ref{app:live-dp}).

\begin{figure}[ht]
\centering
\subfigure{\includegraphics[width=0.45\columnwidth]{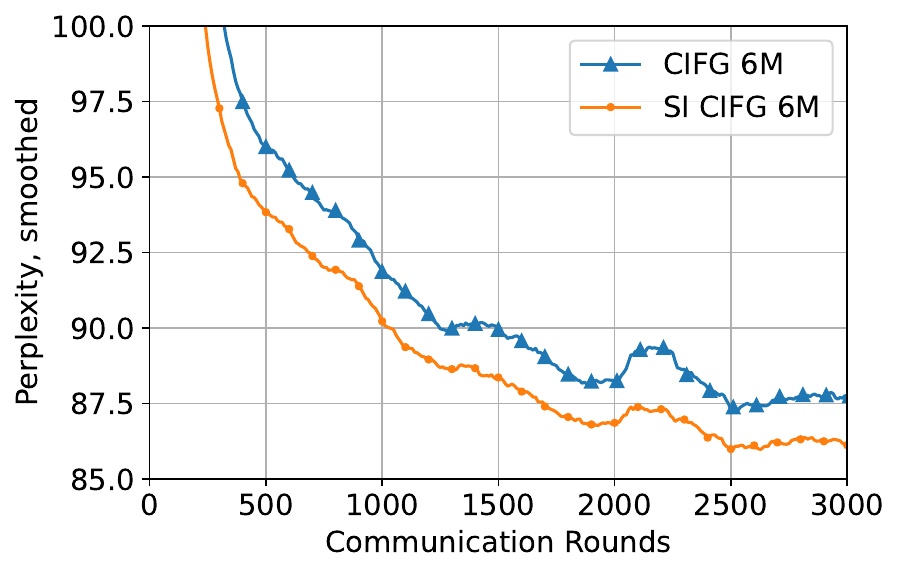}}
\subfigure{\includegraphics[width=0.45\columnwidth]{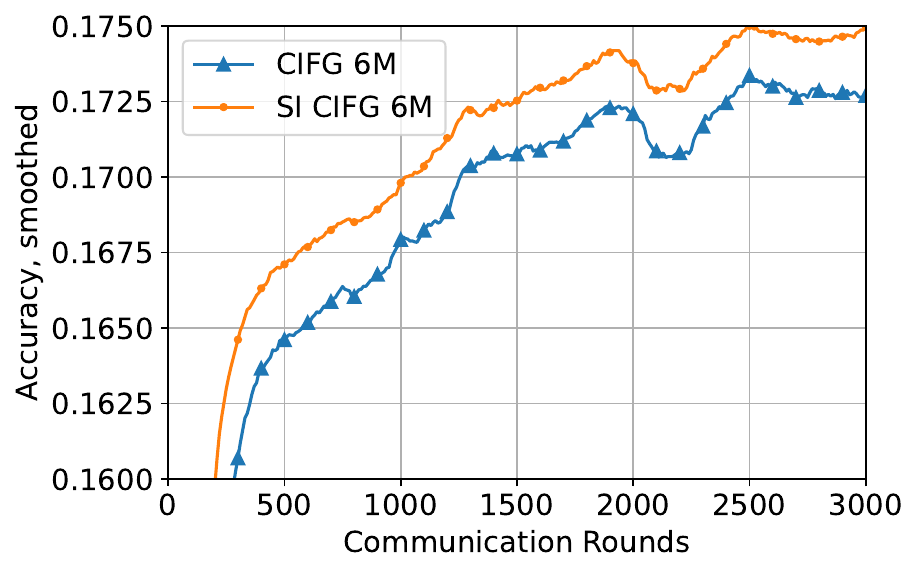}}
\caption{Smoothed perplexity and in-vocab-accuracy from DP live experiments on English virtual keyboard devices.}
\label{fig:live-dp}
\end{figure}

In the DP FL setting, our proposed \newmodel~consistently outperforms the base CIFG and under the same privacy budget, achieves better utility, measured by perplexity and accuracy.

\section{Conclusion}
\label{conclusion}

We applied scale invariance to a variety of neural architectures and proposed a novel CIFG-LSTM architecture (\newmodel) and evaluated their performance on a variety of cross-device and differentially private large scale FL experiments.
We demonstrated that using scale invariant architectures in federated language modeling can significantly accelerate and improve model convergence, with our proposed \newmodel~consistently achieving the best performance and convergence speed.
We hope that this study will motivate further studies into training larger models privately and effectively with federated learning.

\bibliography{refs}

\begin{thebibliography}{48}
\providecommand{\natexlab}[1]{#1}
\providecommand{\url}[1]{\texttt{#1}}
\expandafter\ifx\csname urlstyle\endcsname\relax
  \providecommand{\doi}[1]{doi: #1}\else
  \providecommand{\doi}{doi: \begingroup \urlstyle{rm}\Url}\fi

\bibitem[Acar et~al.(2021)Acar, Zhao, Navarro, Mattina, Whatmough, and Saligrama]{acar2021federated}
Durmus Alp~Emre Acar, Yue Zhao, Ramon~Matas Navarro, Matthew Mattina, Paul~N Whatmough, and Venkatesh Saligrama.
\newblock Federated learning based on dynamic regularization.
\newblock \emph{arXiv preprint arXiv:2111.04263}, 2021.

\bibitem[Alistarh et~al.(2017)Alistarh, Grubic, Li, Tomioka, and Vojnovic]{NIPS2017_6c340f25}
Dan Alistarh, Demjan Grubic, Jerry Li, Ryota Tomioka, and Milan Vojnovic.
\newblock Qsgd: Communication-efficient sgd via gradient quantization and encoding.
\newblock In I.~Guyon, U.~Von Luxburg, S.~Bengio, H.~Wallach, R.~Fergus, S.~Vishwanathan, and R.~Garnett (eds.), \emph{Advances in Neural Information Processing Systems}, volume~30. Curran Associates, Inc., 2017.
\newblock URL \url{https://proceedings.neurips.cc/paper_files/paper/2017/file/6c340f25839e6acdc73414517203f5f0-Paper.pdf}.

\bibitem[Bonawitz et~al.(2022)Bonawitz, Kairouz, Mcmahan, and Ramage]{bonawitz2022federated}
Kallista Bonawitz, Peter Kairouz, Brendan Mcmahan, and Daniel Ramage.
\newblock Federated learning and privacy.
\newblock \emph{Communications of the ACM}, 65\penalty0 (4):\penalty0 90--97, 2022.

\bibitem[Bun \& Steinke(2016)Bun and Steinke]{bun2016concentrated}
Mark Bun and Thomas Steinke.
\newblock Concentrated differential privacy: Simplifications, extensions, and lower bounds.
\newblock In \emph{Theory of Cryptography Conference}, pp.\  635--658. Springer, 2016.

\bibitem[Chen et~al.(2019)Chen, Suresh, Mathews, Wong, Allauzen, Beaufays, and Riley]{chen-etal-2019-federated}
Mingqing Chen, Ananda~Theertha Suresh, Rajiv Mathews, Adeline Wong, Cyril Allauzen, Fran{\c{c}}oise Beaufays, and Michael Riley.
\newblock Federated learning of n-gram language models.
\newblock In \emph{Proceedings of the 23rd Conference on Computational Natural Language Learning (CoNLL)}, pp.\  121--130, Hong Kong, China, November 2019. Association for Computational Linguistics.
\newblock \doi{10.18653/v1/K19-1012}.
\newblock URL \url{https://aclanthology.org/K19-1012}.

\bibitem[Cho et~al.(2014)Cho, van Merri{\"e}nboer, Gulcehre, Bahdanau, Bougares, Schwenk, and Bengio]{cho-etal-2014-learning}
Kyunghyun Cho, Bart van Merri{\"e}nboer, Caglar Gulcehre, Dzmitry Bahdanau, Fethi Bougares, Holger Schwenk, and Yoshua Bengio.
\newblock Learning phrase representations using {RNN} encoder{--}decoder for statistical machine translation.
\newblock In \emph{Proceedings of the 2014 Conference on Empirical Methods in Natural Language Processing ({EMNLP})}, pp.\  1724--1734, Doha, Qatar, October 2014. Association for Computational Linguistics.
\newblock \doi{10.3115/v1/D14-1179}.
\newblock URL \url{https://aclanthology.org/D14-1179}.

\bibitem[Choquette-Choo et~al.(2023)Choquette-Choo, Ganesh, McKenna, McMahan, Rush, Thakurta, and Xu]{choquette2023amplified}
Christopher~A Choquette-Choo, Arun Ganesh, Ryan McKenna, H~Brendan McMahan, Keith Rush, Abhradeep~Guha Thakurta, and Zheng Xu.
\newblock (amplified) banded matrix factorization: A unified approach to private training.
\newblock \emph{arXiv preprint arXiv:2306.08153}, 2023.

\bibitem[Dwork et~al.(2006)Dwork, McSherry, Nissim, and Smith]{dwork2006calibrating}
Cynthia Dwork, Frank McSherry, Kobbi Nissim, and Adam Smith.
\newblock Calibrating noise to sensitivity in private data analysis.
\newblock In \emph{Theory of cryptography conference}, 2006.

\bibitem[Eichner et~al.(2019)Eichner, Koren, McMahan, Srebro, and Talwar]{eichner2019semi}
Hubert Eichner, Tomer Koren, Brendan McMahan, Nathan Srebro, and Kunal Talwar.
\newblock Semi-cyclic stochastic gradient descent.
\newblock In \emph{International Conference on Machine Learning}, pp.\  1764--1773. PMLR, 2019.

\bibitem[El~Ouadrhiri \& Abdelhadi(2022)El~Ouadrhiri and Abdelhadi]{el2022differential}
Ahmed El~Ouadrhiri and Ahmed Abdelhadi.
\newblock Differential privacy for deep and federated learning: A survey.
\newblock \emph{IEEE access}, 10:\penalty0 22359--22380, 2022.

\bibitem[Girgis et~al.(2021)Girgis, Data, Diggavi, Kairouz, and Suresh]{girgis2021shuffled}
Antonious Girgis, Deepesh Data, Suhas Diggavi, Peter Kairouz, and Ananda~Theertha Suresh.
\newblock Shuffled model of differential privacy in federated learning.
\newblock In \emph{International Conference on Artificial Intelligence and Statistics}, pp.\  2521--2529. PMLR, 2021.

\bibitem[Greff et~al.(2017)Greff, Srivastava, Koutník, Steunebrink, and Schmidhuber]{7508408}
Klaus Greff, Rupesh~K. Srivastava, Jan Koutník, Bas~R. Steunebrink, and Jürgen Schmidhuber.
\newblock Lstm: A search space odyssey.
\newblock \emph{IEEE Transactions on Neural Networks and Learning Systems}, 28\penalty0 (10):\penalty0 2222--2232, 2017.
\newblock \doi{10.1109/TNNLS.2016.2582924}.

\bibitem[Hard et~al.(2018{\natexlab{a}})Hard, Kiddon, Ramage, Beaufays, Eichner, Rao, Mathews, and Augenstein]{47586}
Andrew Hard, Chloé~M Kiddon, Daniel Ramage, Francoise Beaufays, Hubert Eichner, Kanishka Rao, Rajiv Mathews, and Sean Augenstein.
\newblock Federated learning for mobile keyboard prediction, 2018{\natexlab{a}}.
\newblock URL \url{https://arxiv.org/abs/1811.03604}.

\bibitem[Hard et~al.(2018{\natexlab{b}})Hard, Rao, Mathews, Beaufays, Augenstein, Eichner, Kiddon, and Ramage]{Hard2018FederatedLF}
Andrew Hard, Kanishka Rao, Rajiv Mathews, Françoise Beaufays, Sean Augenstein, Hubert Eichner, Chlo{\'e} Kiddon, and Daniel Ramage.
\newblock Federated learning for mobile keyboard prediction.
\newblock \emph{ArXiv}, abs/1811.03604, 2018{\natexlab{b}}.
\newblock URL \url{https://api.semanticscholar.org/CorpusID:53207681}.

\bibitem[Hilmkil et~al.(2021)Hilmkil, Callh, Barbieri, S{\"u}tfeld, Zec, and Mogren]{Hilmkil2021ScalingFL}
Agrin Hilmkil, Sebastian Callh, Matteo Barbieri, Leon~Ren{\'e} S{\"u}tfeld, Edvin~Listo Zec, and Olof Mogren.
\newblock Scaling federated learning for fine-tuning of large language models.
\newblock In \emph{International Conference on Applications of Natural Language to Data Bases}, 2021.

\bibitem[Hochreiter \& Schmidhuber(1997{\natexlab{a}})Hochreiter and Schmidhuber]{hochreiter1997}
Sepp Hochreiter and J{\"u}rgen Schmidhuber.
\newblock Long short-term memory.
\newblock \emph{Neural computation}, 9\penalty0 (8):\penalty0 1735--1780, 1997{\natexlab{a}}.

\bibitem[Hochreiter \& Schmidhuber(1997{\natexlab{b}})Hochreiter and Schmidhuber]{hochreiter1997long}
Sepp Hochreiter and J{\"u}rgen Schmidhuber.
\newblock Long short-term memory.
\newblock \emph{Neural computation}, 9\penalty0 (8):\penalty0 1735--1780, 1997{\natexlab{b}}.

\bibitem[Kairouz et~al.(2019)Kairouz, McMahan, Avent, Bellet, Bennis, Bhagoji, Bonawitz, Charles, Cormode, Cummings, D'Oliveira, Rouayheb, Evans, Gardner, Garrett, Gascón, Ghazi, Gibbons, Gruteser, Harchaoui, He, He, Huo, Hutchinson, Hsu, Jaggi, Javidi, Joshi, Khodak, Konečný, Korolova, Koushanfar, Koyejo, Lepoint, Liu, Mittal, Mohri, Nock, Özgür, Pagh, Raykova, Qi, Ramage, Raskar, Song, Song, Stich, Sun, Suresh, Tramèr, Vepakomma, Wang, Xiong, Xu, Yang, Yu, Yu, and Zhao]{49232}
Peter Kairouz, H.~Brendan McMahan, Brendan Avent, Aurélien Bellet, Mehdi Bennis, Arjun~Nitin Bhagoji, K.~A. Bonawitz, Zachary Charles, Graham Cormode, Rachel Cummings, Rafael~G.L. D'Oliveira, Salim~El Rouayheb, David Evans, Josh Gardner, Zachary Garrett, Adrià Gascón, Badih Ghazi, Phillip~B. Gibbons, Marco Gruteser, Zaid Harchaoui, Chaoyang He, Lie He, Zhouyuan Huo, Ben Hutchinson, Justin Hsu, Martin Jaggi, Tara Javidi, Gauri Joshi, Mikhail Khodak, Jakub Konečný, Aleksandra Korolova, Farinaz Koushanfar, Sanmi Koyejo, Tancrède Lepoint, Yang Liu, Prateek Mittal, Mehryar Mohri, Richard Nock, Ayfer Özgür, Rasmus Pagh, Mariana Raykova, Hang Qi, Daniel Ramage, Ramesh Raskar, Dawn Song, Weikang Song, Sebastian~U. Stich, Ziteng Sun, Ananda~Theertha Suresh, Florian Tramèr, Praneeth Vepakomma, Jianyu Wang, Li~Xiong, Zheng Xu, Qiang Yang, Felix~X. Yu, Han Yu, and Sen Zhao.
\newblock Advances and open problems in federated learning.
\newblock 2019.
\newblock URL \url{https://arxiv.org/abs/1912.04977}.

\bibitem[Kairouz et~al.(2021{\natexlab{a}})Kairouz, Mcmahan, Song, Thakkar, Thakurta, and Xu]{pmlr-v139-kairouz21b}
Peter Kairouz, Brendan Mcmahan, Shuang Song, Om~Thakkar, Abhradeep Thakurta, and Zheng Xu.
\newblock Practical and private (deep) learning without sampling or shuffling.
\newblock In Marina Meila and Tong Zhang (eds.), \emph{Proceedings of the 38th International Conference on Machine Learning}, volume 139 of \emph{Proceedings of Machine Learning Research}, pp.\  5213--5225. PMLR, 18--24 Jul 2021{\natexlab{a}}.
\newblock URL \url{https://proceedings.mlr.press/v139/kairouz21b.html}.

\bibitem[Kairouz et~al.(2021{\natexlab{b}})Kairouz, McMahan, Avent, Bellet, Bennis, Bhagoji, Bonawitz, Charles, Cormode, Cummings, et~al.]{kairouz2021advances}
Peter Kairouz, H~Brendan McMahan, Brendan Avent, Aur{\'e}lien Bellet, Mehdi Bennis, Arjun~Nitin Bhagoji, Kallista Bonawitz, Zachary Charles, Graham Cormode, Rachel Cummings, et~al.
\newblock Advances and open problems in federated learning.
\newblock \emph{Foundations and Trends in Machine Learning}, 14\penalty0 (1--2):\penalty0 1--210, 2021{\natexlab{b}}.

\bibitem[Karimireddy et~al.(2021)Karimireddy, Jaggi, Kale, Mohri, Reddi, Stich, and Suresh]{NEURIPS2021_f0e6be4c}
Sai~Praneeth Karimireddy, Martin Jaggi, Satyen Kale, Mehryar Mohri, Sashank Reddi, Sebastian~U Stich, and Ananda~Theertha Suresh.
\newblock Breaking the centralized barrier for cross-device federated learning.
\newblock In M.~Ranzato, A.~Beygelzimer, Y.~Dauphin, P.S. Liang, and J.~Wortman Vaughan (eds.), \emph{Advances in Neural Information Processing Systems}, volume~34, pp.\  28663--28676. Curran Associates, Inc., 2021.
\newblock URL \url{https://proceedings.neurips.cc/paper_files/paper/2021/file/f0e6be4ce76ccfa73c5a540d992d0756-Paper.pdf}.

\bibitem[Kingma \& Ba(2014)Kingma and Ba]{Kingma2014AdamAM}
Diederik~P. Kingma and Jimmy Ba.
\newblock Adam: A method for stochastic optimization.
\newblock \emph{CoRR}, abs/1412.6980, 2014.

\bibitem[Konecn{\'y} et~al.(2016)Konecn{\'y}, McMahan, Yu, Richt{\'a}rik, Suresh, and Bacon]{Konecn2016FederatedLS}
Jakub Konecn{\'y}, H.~B. McMahan, Felix~X. Yu, Peter Richt{\'a}rik, Ananda~Theertha Suresh, and Dave Bacon.
\newblock Federated learning: Strategies for improving communication efficiency.
\newblock \emph{ArXiv}, abs/1610.05492, 2016.

\bibitem[Li et~al.(2020{\natexlab{a}})Li, Sahu, Talwalkar, and Smith]{9084352}
Tian Li, Anit~Kumar Sahu, Ameet Talwalkar, and Virginia Smith.
\newblock Federated learning: Challenges, methods, and future directions.
\newblock \emph{IEEE Signal Processing Magazine}, 37\penalty0 (3):\penalty0 50--60, 2020{\natexlab{a}}.
\newblock \doi{10.1109/MSP.2020.2975749}.

\bibitem[Li et~al.(2020{\natexlab{b}})Li, Sahu, Zaheer, Sanjabi, Talwalkar, and Smith]{li2020federated}
Tian Li, Anit~Kumar Sahu, Manzil Zaheer, Maziar Sanjabi, Ameet Talwalkar, and Virginia Smith.
\newblock Federated optimization in heterogeneous networks, 2020{\natexlab{b}}.

\bibitem[Li et~al.(2022)Li, Bhojanapalli, Zaheer, Reddi, and Kumar]{pmlr-v162-li22b}
Zhiyuan Li, Srinadh Bhojanapalli, Manzil Zaheer, Sashank Reddi, and Sanjiv Kumar.
\newblock Robust training of neural networks using scale invariant architectures.
\newblock In Kamalika Chaudhuri, Stefanie Jegelka, Le~Song, Csaba Szepesvari, Gang Niu, and Sivan Sabato (eds.), \emph{Proceedings of the 39th International Conference on Machine Learning}, volume 162 of \emph{Proceedings of Machine Learning Research}, pp.\  12656--12684. PMLR, 17--23 Jul 2022.
\newblock URL \url{https://proceedings.mlr.press/v162/li22b.html}.

\bibitem[McMahan et~al.(2017{\natexlab{a}})McMahan, Moore, Ramage, Hampson, and Arcas]{pmlr-v54-mcmahan17a}
Brendan McMahan, Eider Moore, Daniel Ramage, Seth Hampson, and Blaise Aguera~y Arcas.
\newblock {Communication-Efficient Learning of Deep Networks from Decentralized Data}.
\newblock In Aarti Singh and Jerry Zhu (eds.), \emph{Proceedings of the 20th International Conference on Artificial Intelligence and Statistics}, volume~54 of \emph{Proceedings of Machine Learning Research}, pp.\  1273--1282. PMLR, 20--22 Apr 2017{\natexlab{a}}.
\newblock URL \url{https://proceedings.mlr.press/v54/mcmahan17a.html}.

\bibitem[McMahan et~al.(2017{\natexlab{b}})McMahan, Ramage, Talwar, and Zhang]{mcmahan2017learning}
H~Brendan McMahan, Daniel Ramage, Kunal Talwar, and Li~Zhang.
\newblock Learning differentially private recurrent language models.
\newblock \emph{arXiv preprint arXiv:1710.06963}, 2017{\natexlab{b}}.

\bibitem[Pascanu et~al.(2013)Pascanu, Mikolov, and Bengio]{pmlr-v28-pascanu13}
Razvan Pascanu, Tomas Mikolov, and Yoshua Bengio.
\newblock On the difficulty of training recurrent neural networks.
\newblock In Sanjoy Dasgupta and David McAllester (eds.), \emph{Proceedings of the 30th International Conference on Machine Learning}, volume~28 of \emph{Proceedings of Machine Learning Research}, pp.\  1310--1318, Atlanta, Georgia, USA, 17--19 Jun 2013. PMLR.
\newblock URL \url{https://proceedings.mlr.press/v28/pascanu13.html}.

\bibitem[Press \& Wolf(2017)Press and Wolf]{press-wolf-2017-using}
Ofir Press and Lior Wolf.
\newblock Using the output embedding to improve language models.
\newblock In \emph{Proceedings of the 15th Conference of the {E}uropean Chapter of the Association for Computational Linguistics: Volume 2, Short Papers}, pp.\  157--163, Valencia, Spain, April 2017. Association for Computational Linguistics.
\newblock URL \url{https://aclanthology.org/E17-2025}.

\bibitem[Raffel et~al.(2019)Raffel, Shazeer, Roberts, Lee, Narang, Matena, Zhou, Li, and Liu]{2019t5}
Colin Raffel, Noam Shazeer, Adam Roberts, Katherine Lee, Sharan Narang, Michael Matena, Yanqi Zhou, Wei Li, and Peter~J. Liu.
\newblock Exploring the limits of transfer learning with a unified text-to-text transformer.
\newblock \emph{arXiv e-prints}, 2019.

\bibitem[Ramaswamy et~al.(2020)Ramaswamy, Thakkar, Mathews, Andrew, McMahan, and Beaufays]{ramaswamy2020training}
Swaroop Ramaswamy, Om~Thakkar, Rajiv Mathews, Galen Andrew, H~Brendan McMahan, and Fran{\c{c}}oise Beaufays.
\newblock Training production language models without memorizing user data.
\newblock \emph{arXiv preprint arXiv:2009.10031}, 2020.

\bibitem[Reddi et~al.(2020)Reddi, Charles, Zaheer, Garrett, Rush, Kone{\v{c}}n{\`y}, Kumar, and McMahan]{reddi2020adaptive}
Sashank Reddi, Zachary Charles, Manzil Zaheer, Zachary Garrett, Keith Rush, Jakub Kone{\v{c}}n{\`y}, Sanjiv Kumar, and H~Brendan McMahan.
\newblock Adaptive federated optimization.
\newblock \emph{arXiv preprint arXiv:2003.00295}, 2020.

\bibitem[Reddi et~al.(2021)Reddi, Charles, Zaheer, Garrett, Rush, Konečný, Kumar, and McMahan]{50448}
Sashank Reddi, Zachary~Burr Charles, Manzil Zaheer, Zachary Garrett, Keith Rush, Jakub Konečný, Sanjiv Kumar, and Brendan McMahan (eds.).
\newblock \emph{Adaptive Federated Optimization}, 2021.
\newblock URL \url{https://openreview.net/forum?id=LkFG3lB13U5}.

\bibitem[Ro et~al.(2021{\natexlab{a}})Ro, Chen, Mathews, Mohri, and Suresh]{Ro2021CommunicationEfficientAF}
Jae Ro, Mingqing Chen, Rajiv Mathews, Mehryar Mohri, and Ananda~Theertha Suresh.
\newblock Communication-efficient agnostic federated averaging.
\newblock In \emph{Interspeech}, 2021{\natexlab{a}}.

\bibitem[Ro et~al.(2022)Ro, Breiner, McConnaughey, Chen, Suresh, Kumar, and Mathews]{ro-etal-2022-scaling}
Jae Ro, Theresa Breiner, Lara McConnaughey, Mingqing Chen, Ananda Suresh, Shankar Kumar, and Rajiv Mathews.
\newblock Scaling language model size in cross-device federated learning.
\newblock In \emph{Proceedings of the First Workshop on Federated Learning for Natural Language Processing (FL4NLP 2022)}, pp.\  6--20, Dublin, Ireland, May 2022. Association for Computational Linguistics.
\newblock \doi{10.18653/v1/2022.fl4nlp-1.2}.
\newblock URL \url{https://aclanthology.org/2022.fl4nlp-1.2}.

\bibitem[Ro et~al.(2021{\natexlab{b}})Ro, Suresh, and Wu]{fedjax2021}
Jae~Hun Ro, Ananda~Theertha Suresh, and Ke~Wu.
\newblock {F}ed{JAX}: Federated learning simulation with {JAX}.
\newblock \emph{arXiv preprint arXiv:2108.02117}, 2021{\natexlab{b}}.

\bibitem[Shah et~al.(2020)Shah, Hard, Nguyen, Moreno, Partridge, Subrahmanya, Zhu, and Mathews]{49696}
Aishanee Shah, Andrew Hard, Cameron Nguyen, Ignacio~Lopez Moreno, Kurt Partridge, Niranjan Subrahmanya, Pai Zhu, and Rajiv Mathews.
\newblock Training keyword spotting models on non-iid data with federated learning.
\newblock In \emph{Interspeech}, 2020.

\bibitem[Song et~al.(2021)Song, Salcianu, Song, Dopson, and Zhou]{song-etal-2021-fast}
Xinying Song, Alex Salcianu, Yang Song, Dave Dopson, and Denny Zhou.
\newblock Fast {W}ord{P}iece tokenization.
\newblock In \emph{Proceedings of the 2021 Conference on Empirical Methods in Natural Language Processing}, pp.\  2089--2103, Online and Punta Cana, Dominican Republic, November 2021. Association for Computational Linguistics.
\newblock \doi{10.18653/v1/2021.emnlp-main.160}.
\newblock URL \url{https://aclanthology.org/2021.emnlp-main.160}.

\bibitem[Suresh et~al.(2017)Suresh, Yu, Kumar, and McMahan]{pmlr-v70-suresh17a}
Ananda~Theertha Suresh, Felix~X. Yu, Sanjiv Kumar, and H.~Brendan McMahan.
\newblock Distributed mean estimation with limited communication.
\newblock In Doina Precup and Yee~Whye Teh (eds.), \emph{Proceedings of the 34th International Conference on Machine Learning}, volume~70 of \emph{Proceedings of Machine Learning Research}, pp.\  3329--3337. PMLR, 06--11 Aug 2017.
\newblock URL \url{https://proceedings.mlr.press/v70/suresh17a.html}.

\bibitem[TFF(2018)]{tff}
TFF.
\newblock Tensorflow federated, 2018.
\newblock URL \url{https://www.tensorflow.org/federated}.

\bibitem[Vaswani et~al.(2017)Vaswani, Shazeer, Parmar, Uszkoreit, Jones, Gomez, Kaiser, and Polosukhin]{vaswani2017}
Ashish Vaswani, Noam Shazeer, Niki Parmar, Jakob Uszkoreit, Llion Jones, Aidan~N Gomez, \L~ukasz Kaiser, and Illia Polosukhin.
\newblock Attention is all you need.
\newblock In I.~Guyon, U.~V. Luxburg, S.~Bengio, H.~Wallach, R.~Fergus, S.~Vishwanathan, and R.~Garnett (eds.), \emph{Advances in Neural Information Processing Systems}, volume~30. Curran Associates, Inc., 2017.
\newblock URL \url{https://proceedings.neurips.cc/paper/2017/file/3f5ee243547dee91fbd053c1c4a845aa-Paper.pdf}.

\bibitem[Wei et~al.(2020)Wei, Li, Ding, Ma, Yang, Farokhi, Jin, Quek, and Poor]{wei2020federated}
Kang Wei, Jun Li, Ming Ding, Chuan Ma, Howard~H Yang, Farhad Farokhi, Shi Jin, Tony~QS Quek, and H~Vincent Poor.
\newblock Federated learning with differential privacy: Algorithms and performance analysis.
\newblock \emph{IEEE transactions on information forensics and security}, 15:\penalty0 3454--3469, 2020.

\bibitem[Wu et~al.(2016)Wu, Schuster, Chen, Le, Norouzi, Macherey, Krikun, Cao, Gao, Macherey, Klingner, Shah, Johnson, Liu, Kaiser, Gouws, Kato, Kudo, Kazawa, Stevens, Kurian, Patil, Wang, Young, Smith, Riesa, Rudnick, Vinyals, Corrado, Hughes, and Dean]{Wu2016GooglesNM}
Yonghui Wu, Mike Schuster, Z.~Chen, Quoc~V. Le, Mohammad Norouzi, Wolfgang Macherey, Maxim Krikun, Yuan Cao, Qin Gao, Klaus Macherey, Jeff Klingner, Apurva Shah, Melvin Johnson, Xiaobing Liu, Lukasz Kaiser, Stephan Gouws, Yoshikiyo Kato, Taku Kudo, Hideto Kazawa, Keith Stevens, George Kurian, Nishant Patil, Wei Wang, Cliff Young, Jason~R. Smith, Jason Riesa, Alex Rudnick, Oriol Vinyals, Gregory~S. Corrado, Macduff Hughes, and Jeffrey Dean.
\newblock Google's neural machine translation system: Bridging the gap between human and machine translation.
\newblock \emph{ArXiv}, abs/1609.08144, 2016.

\bibitem[Xu et~al.(2023)Xu, Zhang, Andrew, Choquette, Kairouz, Mcmahan, Rosenstock, and Zhang]{xu-etal-2023-federated}
Zheng Xu, Yanxiang Zhang, Galen Andrew, Christopher Choquette, Peter Kairouz, Brendan Mcmahan, Jesse Rosenstock, and Yuanbo Zhang.
\newblock Federated learning of gboard language models with differential privacy.
\newblock In Sunayana Sitaram, Beata Beigman~Klebanov, and Jason~D Williams (eds.), \emph{Proceedings of the 61st Annual Meeting of the Association for Computational Linguistics (Volume 5: Industry Track)}, pp.\  629--639, Toronto, Canada, July 2023. Association for Computational Linguistics.
\newblock \doi{10.18653/v1/2023.acl-industry.60}.
\newblock URL \url{https://aclanthology.org/2023.acl-industry.60}.

\bibitem[Yu et~al.(2019)Yu, Si, Hu, and Zhang]{yu2019review}
Yong Yu, Xiaosheng Si, Changhua Hu, and Jianxun Zhang.
\newblock A review of recurrent neural networks: Lstm cells and network architectures.
\newblock \emph{Neural computation}, 31\penalty0 (7):\penalty0 1235--1270, 2019.

\bibitem[Zhang et~al.(2020)Zhang, Karimireddy, Veit, Kim, Reddi, Kumar, and Sra]{zhang2020adaptive}
Jingzhao Zhang, Sai~Praneeth Karimireddy, Andreas Veit, Seungyeon Kim, Sashank Reddi, Sanjiv Kumar, and Suvrit Sra.
\newblock Why are adaptive methods good for attention models?
\newblock \emph{Advances in Neural Information Processing Systems}, 33:\penalty0 15383--15393, 2020.

\bibitem[Zhang et~al.(2023)Zhang, Ramage, Xu, Zhang, Zhai, and Kairouz]{zhang2023private}
Yuanbo Zhang, Daniel Ramage, Zheng Xu, Yanxiang Zhang, Shumin Zhai, and Peter Kairouz.
\newblock Private federated learning in gboard.
\newblock \emph{arXiv preprint arXiv:2306.14793}, 2023.

\end{thebibliography}
\bibliographystyle{iclr2024_conference}
% %%%%%%%%%%%%%%%%%%%%%%%%%%%%%%%%%%%%%%%%%%%%%%%%%%%%%%%%%%%%%%%%%%%%%%%%%%%%%%%
% %%%%%%%%%%%%%%%%%%%%%%%%%%%%%%%%%%%%%%%%%%%%%%%%%%%%%%%%%%%%%%%%%%%%%%%%%%%%%%%
% % APPENDIX
% %%%%%%%%%%%%%%%%%%%%%%%%%%%%%%%%%%%%%%%%%%%%%%%%%%%%%%%%%%%%%%%%%%%%%%%%%%%%%%%
% %%%%%%%%%%%%%%%%%%%%%%%%%%%%%%%%%%%%%%%%%%%%%%%%%%%%%%%%%%%%%%%%%%%%%%%%%%%%%%%
\appendix

\section{Federated experiments on public datasets details}\label{app:simulation}

For all models and experiments with the Stack Overflow federated dataset, we used the followed fixed hyperparameters

\begin{itemize}
  \item Number of clients per round = $500$: Number of clients sampled per communication round of FL training.
  \item Client batch size = $10$: Batch size used during local steps of training on client data.
  \item Number of client epochs = $1$: Number of epochs of training on client data.
  \item Number of client batches = $120$: Maximum number of client batches to train on until number of client epochs is reached.
  \item Maximum sequence length = $20$: Maximum allowed sequence length. Shorter sequences are padded and longer sequences are truncated to this.
  \item Client optimizer = SGD
  \item Server optimizer = Adam with $\beta_1$ at $0.9$, $\beta_2$ at $0.999$, and epsilon at $1e^{-8}$.
\end{itemize}

Table~\ref{tab:stackoverflow-hyper-sweep} details the hyperparameter configurations swept over per model, where the selected hyperparameters were chosen based on the lowest loss on the heldout split of the Stack Overflow federated dataset after $3K$ rounds of training averaged over 5 random seeds.
Table~\ref{tab:stackoverflow} reports the final evaluation results using these selected hyperparameters on the Stack Overflow test dataset.

\begin{table}[t]
\centering
\caption{Selected hyperparameters for each model.
The values in $[\ ]$ are the possible hyperparameter values searched over.}
\begin{tabular}{ccc}
& Client learning rate & Server learning rate \\
Model & $[0.1, 0.5, 1.0, 2.0]$ & $[0.001, 0.01]$ \\
\hline
CIFG $19$M & $0.1$ & $0.001$ \\
\newmodel~$19$M & $0.1$ & $0.001$ \\
Transformer $21$M & $0.5$ & $0.001$ \\
SI Transformer $21$M & $2.0$ & $0.01$ \\
\end{tabular}
\label{tab:stackoverflow-hyper-sweep}
\end{table}

\begin{table}[tt]
\centering
\caption{Perplexity and accuracy on the Stack Overflow test dataset after $3K$ communication rounds.}
\vspace{.1in}
\begin{tabular}{c c c}
\hline
Model & Perplexity & Accuracy\% \\
\hline
CIFG $19$M & $35.5\pm0.2$ & $33.1\pm0.1$ \\
\newmodel~$19$M & $\mathbf{33.6\pm0.2}$ & $\mathbf{33.6\pm0.1}$ \\
Transformer $21$M & $34.6\pm0.1$ & $33.4\pm0.0$ \\
SI Transformer $21$M & $33.7\pm0.1$ & $33.5\pm0.0$ \\
\hline
\end{tabular}
\label{tab:stackoverflow}
\end{table}

\section{Live production experiment details}\label{app:live}

For all models and experiments with the live English virtual keyboard user population, we used the followed fixed hyperparameters.
We note that due to the nature of live production experiments and longer feedback times, we were not able to run any extensive hyperparameter sweeps and re-used many common settings used in previous experiments.
Table~\ref{tab:live} reports the final evaluation results using these hyperparameters averaged over the final $100$ of $3K$ communication rounds to account for daytime variability \citep{eichner2019semi}.

\begin{itemize}
  \item Number of clients per round = $500$: Number of clients sampled per communication round of FL training.
  \item Client batch size = $10$: Batch size used during local steps of training on client data.
  \item Number of client epochs = $1$: Number of epochs of training on client data.
  \item Number of client batches = $120$: Maximum number of client batches to train on until number of client epochs is reached.
  \item Maximum sequence length = $20$: Maximum allowed sequence length.
  \item Client optimizer = SGD with learning rate $0.7$.
  \item Server optimizer = Adam with learning rate $0.02$, $\beta_1$ at $0.9$, $\beta_2$ at $0.999$, and epsilon at $1e^{-8}$.
\end{itemize}

\begin{table}[ht]
\centering
\caption{Perplexity and accuracy from live experiments on English virtual keyboard devices averaged with standard deviations over the final $100$ communication rounds. $^*$For base CIFG, we use the last $100$ rounds before divergence. }
\vspace{.1in}
\begin{tabular}{c c c}
\hline
Model & Perplexity & Accuracy\% \\
\hline
$^*$CIFG $9$M & $47.5\pm1.1$ & $21.8\pm0.2$ \\
\newmodel~$9$M & $\mathbf{42.2\pm0.2}$ & $\mathbf{23.4\pm0.1}$ \\
Transformer $11$M & $63.6\pm0.9$ & $18.2\pm0.1$ \\
SI Transformer $11$M & $44.3\pm0.7$ & $23.2\pm0.2$ \\
\hline
\end{tabular}
\label{tab:live}
\end{table}

\section{Experiments with differentially private federated learning details}\label{app:live-dp}

For all models and DP experiments with the live English virtual keyboard user population, we used the followed fixed hyperparameters.
Again, due to the nature of live production experiments and longer feedback times, we were not able to run any extensive hyperparameter sweeps and re-used many common settings used in previous experiments.
Table~\ref{tab:live-dp} reports the final evaluation results using these hyperparameters averaged over the final $100$ of $3K$ communication rounds to account for daytime variability.

\begin{itemize}
  \item Number of clients per round = $6500$: Number of clients sampled per communication round of FL training.
  \item Client batch size = $10$: Batch size used during local steps of training on client data.
  \item Clipping norm = $5.0$: Fixed L2 norm that client updates are clipped up to.
  \item Maximum sequence length = $10$: Maximum allowed sequence length. Decreased here since word tokenization is used instead of the typically longer subword tokenization.
  \item Client optimizer = SGD with learning rate $0.5$.
  \item Server optimizer = SGD with momentum with learning rate $1.0$ and momentum $0.9$.
\end{itemize}

\begin{table}[ht]
\centering
\caption{Perplexity and in-vocab-accuracy from DP live experiments on English virtual keyboard devices averaged with standard deviations over the final $100$ communication rounds.}
\vspace{.1in}
\begin{tabular}{c c c}
\hline
Model & Perplexity & Accuracy\% \\
\hline
CIFG $6$M & $88.0\pm0.8$ & $17.3\pm0.1$ \\
\newmodel~$6$M & $\mathbf{86.1\pm0.7}$ & $\mathbf{17.5\pm0.1}$ \\
\hline
\end{tabular}
\label{tab:live-dp}
\end{table}

\end{document}